\documentclass{article}

\usepackage{cite}
\usepackage{url}
\usepackage{graphicx}

\usepackage{float}
\usepackage{amssymb}
\usepackage{amsmath}
\usepackage{xfrac}
\usepackage{mathtools}
\usepackage{enumitem}

\usepackage{amsthm}

\newtheorem{lemma}{Lemma}

\newtheorem{theorem}{Theorem}
\newtheorem{definition}{Definition}
\newtheorem{example}{Example}

\newcommand{\ZO}{\{0,1\}}

\newcommand{\varl}[2]{\textup{var}_{#1}\left(#2\right)}

\newcommand{\inv}[1]{#1^{\text{ }\scalebox{0.8}[0.75]{\textup{-1}}}}

\newcommand{\card}[1]{\vert #1 \vert}

\newcommand{\OBDD}{\textsf{OBDD}}
\newcommand{\BDD}{\textsf{BDD}}

\newcommand{\NNF}{\textsf{NNF}}
\newcommand{\DNNF}{\textsf{DNNF}}
\newcommand{\CNF}{\textsf{CNF}}
\newcommand{\AND}{\textsf{AND}}
\newcommand{\OR}{\textsf{OR}}

\usepackage[bookmarks=true, citecolor=black, linkcolor=black, colorlinks=true]{hyperref}

\usepackage{tikz}
\usetikzlibrary{calc}
\usepackage{subcaption}
\usetikzlibrary{patterns}

\usepackage{authblk}

\begin{document}

\title{A Lower Bound on DNNF Encodings of Pseudo-Boolean Constraints}

\author{
	Alexis de Colnet
}

\affil{
	CRIL, CNRS \& Univ Artois
}
\date{ }
\maketitle

\begin{abstract}
Two major considerations when encoding pseudo-Boolean (PB) constraints into SAT are the size of the encoding and its propagation strength, that is, the guarantee that it has a good behaviour under unit propagation. Several encodings with propagation strength guarantees rely upon prior compilation of the constraints into {\DNNF} (decomposable negation normal form), {\BDD} (binary decision diagram), or some other sub-variants. However it has been shown that there exist PB-constraints whose ordered {\BDD} ({\OBDD}) representations, and thus the inferred {\CNF} encodings, all have exponential size. Since {\DNNF}s are more succinct than {\OBDD}s, preferring encodings via {\DNNF} to avoid size explosion seems a legitimate choice. Yet in this paper, we prove the existence of PB-constraints whose {\DNNF}s all require exponential size. 
\end{abstract}

\section{Introduction}
\label{section:introduction}
\par 
Pseudo-Boolean constraints (or PB-constraints) are Boolean functions over $0/1$ Boolean variables $x_1,\dots,x_n$ of the form $\sum_{i = 1}^n w_ix_i\text{ }\texttt{'op'}\text{ }\theta$ where the $w_i$ are integer weights,~$\theta$ is an integer threshold and $\texttt{'op'}$ is a comparison operator $<$, $\leq$, $>$ or $\geq$. PB-constraints have been studied extensively under different names (e.g. threshold functions \cite{HosakaTKY97}, Knapsack constraints \cite{GopalanKMSVV11}) due to their omnipresence in many domains of AI and their wide range of practical applications \cite{Ivuanescu65,Papaioannou77,BorosHMR99,BryantLS02,AloulRMS02}.
\par
One way to handle PB-constraints in a constraint satisfaction problem is to translate them into a {\CNF} formula and feed it to a SAT solver. The general idea is to generate a {\CNF}, possibly introducing auxiliary Boolean variables, whose restriction to variables of the constraint is equivalent to the constraint. Two major considerations here are the size of the {\CNF} encoding and its propagation strength. One wants, on the one hand, to avoid the size of the encoding to explode, and on the other hand, to guarantee a good behaviour of the SAT instance under unit propagation -- a technique at the very core of SAT solving. Desired propagation strength properties are, for instance, generalised arc consistency (GAC) \cite{Bacchus07} or propagation completeness (PC) \cite{BordeauxM12}. Several encodings to {\CNF} follow the same two-steps method: first, each constraint is represented in a compact form such as {\BDD} (Binary Decision Diagram) or {\DNNF} (Decomposable Negation Normal Form). Second, the compact forms are turned into {\CNF}s using Tseitin or other transformations. The SAT instance is the conjunction of all obtained {\CNF}s. It is worth mentioning that there are GAC encodings of PB-constraints into polynomial size {\CNF}s that do not follow this two-steps method~\cite{BailleuxBR09}. However no similar result is known for PC encodings. PC encodings are more restrictive that GAC encodings and may be obtained via techniques requiring compilation to {\DNNF}~\cite{KuceraS19}. Thus the first step is a knowledge compilation task.
\par 
Knowledge compilation studies different representations for knowledge \cite{DarwicheM02,Marquis15} under the general idea that some representations are more suitable than others when solving specific reasoning problems. One observation that has been made is that the more reasoning tasks can be solved efficiently with particular representations, the larger these representations get in size. In the context of constraint encodings to SAT, the conversion of compiled forms to {\CNF}s does not reduce the size of the SAT instance, therefore it is essential to control the size of the representations obtained by knowledge compilation.
\par
Several representations have been studied with respect to different encoding techniques with the purpose of determining which properties of representations are sufficient to ensure propagation strength \cite{EenS06 ,JungBKW08, GangeS12, AbioNORM14, AbioGMS16, KuceraS19}. Popular representations in this context are {\DNNF} and {\BDD} and their many variants: deterministic {\DNNF}, smooth {\DNNF}, ordered {\BDD} ({\OBDD})$\dots$ As mentioned above, a problem occurring when compiling a constraint into such representations is that exponential space may be required. Most notably, it has been shown in \cite{HosakaTKY97,AbioNORM14} that some PB-constraints can only be represented by {\OBDD}s whose size is exponential in $\sqrt{n}$, where $n$ is the number of variables. Our contribution is the proof of the following theorem where we lift the statement from {\OBDD} to {\DNNF}.
\begin{theorem}\label{theorem:PB_constraints_hard_for_DNNF}
There is a class of PB-constraints $\mathcal{F}$ such that for any constraint $f \in \mathcal{F}$ on $n^2$ variables, any {\DNNF} representation of $f$ has size $2^{\Omega(n)}$.
\end{theorem}
Since {\DNNF}s are exponentially more succinct than {\OBDD}s \cite{DarwicheM02}, our result is a generalisation of the result in \cite{HosakaTKY97,AbioNORM14}. The class $\mathcal{F}$ is similar to that used in \cite{HosakaTKY97,AbioNORM14}, actually the only difference is the choice of the threshold for the PB-constraints. Yet, adapting proofs given in \cite{HosakaTKY97, AbioNORM14} for {\OBDD} to {\DNNF} is  not straightforward, thus our proof of Theorem~\ref{theorem:PB_constraints_hard_for_DNNF} bears very little resemblance.
\par
It has been shown in \cite{BerreMMW18} that there exist sets of PB-constraints such that the whole \emph{set} (so a conjunction of PB-constraints) requires exponential size {\DNNF} to represent. Our result is a generalisation to \emph{single} PB-constraints.

\section{Preliminaries}
\label{section:preliminaries}
\paragraph{\textbf{\textup{Conventions of notation.}}} Boolean variables are seen as variables over $\ZO$, where $0$ and $1$ represent $false$ and $true$ respectively. Via this $0/1$ representation, Boolean variables can be used in arithmetic expressions over $\mathbb{Z}$. For notational convenience, we keep the usual operators $\neg$, $\vee$ and $\wedge$ to denote, respectively, the negation, disjunction and conjunction of Boolean variables or functions. Given $X$ a set of $n$ Boolean variables, assignments to $X$ are seen as vectors in $\ZO^n$. Single Boolean variables are written in plain text ($x$) while assignments to several variables are written in bold ($\mathbf{x}$). We write $\mathbf{x} \leq \mathbf{y}$ when the vector $\mathbf{y}$ dominates $\mathbf{x}$ element-wise. We write $\mathbf{x} < \mathbf{y}$ when $\mathbf{x} \leq \mathbf{y}$ and $\mathbf{x} \neq \mathbf{y}$. In this framework, a Boolean function $f$ over $X$ is a mapping from $\ZO^n$ to $\ZO$. $f$ is said to \emph{accept} an assignment $\mathbf{x}$ when $f(\mathbf{x}) = 1$, then $\mathbf{x}$ is called a \emph{model} of $f$. The function is \emph{monotone} if for any model $\mathbf{x}$ of $f$, all $\mathbf{y} \geq \mathbf{x}$ are models of $f$ as well. The set of models of $f$ is denoted $\inv{f}(1)$. Given $f$ and $g$ two Boolean functions over $X$, we write $f \leq g$  when $\inv{f}(1) \subseteq \inv{g}(1)$. We write $f < g$ when the inclusion is strict. 

\paragraph{\textbf{\textup{Pseudo-Boolean constraints.}}}  \emph{Pseudo-Boolean (PB) constraints} are inequalities the form $\sum_{i=1}^n w_i x_i \text{ }\texttt{'op'} \text{ } \theta$ where the $x_i$ are $0/1$ Boolean variables, the~$w_i$ and $\theta$ are integers, and $\texttt{'op'}$ is one of the comparison operator $<$, $\leq$, $>$ or~$\geq$. A PB-constraint is associated with a Boolean function whose models are exactly the assignments to $\{x_1, \dots, x_n\}$ that satisfy the inequality. For simplicity we directly consider PB-constraints as Boolean functions -- although the same function may represent different constraints -- while keeping the term ``constraints'' when referring to them. In this paper, we restrict our attention to PB-constraints where $\texttt{'op'}$ is $\geq$ and all weights are positive integers. Note that such PB-constraints are monotone Boolean functions. Given a sequence of positive integer weights $W = (w_1, \dots, w_n)$ and an integer threshold $\theta$, we define the function $w : \ZO^n \rightarrow \mathbb{N}$ that maps any assignment to its weight by $w(\mathbf{x}) = \sum_{i=1}^n w_i x_i$. With these notations, a PB-constraint over $X$ for a given pair $(W,\theta)$ is a Boolean function whose models are exactly the $\mathbf{x}$ such that $w(\mathbf{x}) \geq \theta$.\begin{example}\label{example:PB_constraint_example}
Let $n = 5$, $W = (1,2,3,4,5)$ and $\theta = 9$. The PB-constraint for $(W,\theta)$ is the Boolean function whose models are the assignments such that $\sum_{i = 1}^5 ix_i \geq~9$. E.g. $\mathbf{x} = (0,1,1,0,1)$ \normalsize is a model of weight $w(\mathbf{x}) = 10$. 
\end{example}
\noindent For notational clarity, given any subset $Y \subseteq X$ and denoting $\mathbf{x}\vert_Y$ the restriction of $\mathbf{x}$ to variables of $Y$, we overload $w$ so that $w(\mathbf{x}\vert_Y)$ is the sum of weights activated by variables of $Y$ set to $1$ in $\mathbf{x}$. 

\paragraph{\textbf{\textup{Decomposable {\NNF}.}}} A circuit in \emph{negation normal form} ({\NNF}) is a single output Boolean circuit whose inputs are Boolean variables and their complements, and whose gates are fanin-2 {\AND} and {\OR} gates. The \emph{size} of the circuit is the number of its gates. We say that an {\NNF} is \emph{decomposable} ({\DNNF}) if for any {\AND} gate, the two sub-circuits rooted at that gate share no input variable, i.e., if $x$ or $\neg x$ is an input of the circuit rooted at the left input of the {\AND} gate, then neither $x$ nor $\neg x$ is an input of the circuit rooted at the right input, and vice versa. A Boolean function $f$ is \emph{encoded} by a {\DNNF}~$D$ if the assignments of variables for which the output of $D$ is $1$ ($true$) are exactly the models of~$f$. 

\paragraph{\textbf{\textup{Rectangle covers.}}} Let $X$ be a finite set of Boolean variables and let $\Pi = (X_1, X_2)$ be a partition of $X$ (i.e., $X_1 \cup X_2 = X$ and $X_1 \cap X_2 = \emptyset$). A \emph{rectangle}~$r$ with respect to $\Pi$ is a Boolean function over $X$ defined as the conjunction of two functions $\rho_1$ and $\rho_2$ over $X_1$ and $X_2$ respectively. $\Pi$ is called the \emph{partition} of $r$. We say that the partition and the rectangle are \emph{balanced} when $\frac{\card{X}}{3} \leq \card{X_1} \leq \frac{2\card{X}}{3}$ (thus the same holds for $X_2$). Whenever considering a partition $(X_1,X_2)$, we use for any assignment $\mathbf{x}$ to $X$ the notations $\mathbf{x}_1 \coloneqq \mathbf{x}\vert_{X_1}$ and $\mathbf{x}_2 \coloneqq \mathbf{x}\vert_{X_2}$. And for any two assignments $\mathbf{x}_1$ and $\mathbf{x}_2$ to $X_1$ and $X_2$, we note $(\mathbf{x}_1, \mathbf{x}_2)$ the assignment to $X$ whose restrictions to $X_1$ and $X_2$ are $\mathbf{x}_1$ and $\mathbf{x}_2$. Given $f$ a Boolean function over $X$, a \emph{rectangle cover} of $f$ is a disjunction of rectangles over $X$, possibly with different partitions, equivalent to $f$. 
The \emph{size} of a rectangle cover is the number of its rectangles. A cover is called \emph{balanced} if all its rectangles are balanced.\begin{example}\label{example:rectangle_example}
Going back to Example~\ref{example:PB_constraint_example}, consider the partition $X_1 \coloneqq  \{x_1,x_3,x_4\}$, $X_2 \coloneqq  \{x_2, x_5\}$ and define $\rho_1 \coloneqq x_3 \wedge x_4$ and $\rho_2 \coloneqq x_2 \vee x_5$. Then $r \coloneqq \rho_1 \wedge \rho_2$ is a rectangle w.r.t. this partition that accepts only models of the PB-constraint from Example~\ref{example:PB_constraint_example}. Thus it can be part of a rectangle cover for this constraint.
\end{example} 
\noindent Any function $f$ has at least one balanced rectangle cover as one can create a balanced rectangle accepting exactly one chosen model of $f$. We denote by $C(f)$ the size of the smallest balanced rectangle cover of $f$. The following result from \cite{BovaCMS16} links $C(f)$ to the size of any {\DNNF} encoding $f$.\begin{theorem}\label{theorem:size_rectangle_cover_size_dnnf}
Let $D$ be a \textsf{DNNF} encoding a Boolean function $f$. Then $f$ has a balanced rectangle cover of size at most the size of $D$. 
\end{theorem} 
\noindent Theorem~\ref{theorem:size_rectangle_cover_size_dnnf} reduces the problem of finding lower bounds on the size of {\DNNF}s encoding $f$ to that of finding lower bounds on $C(f)$.

\section{Restriction to Threshold Models of PB-Constraints}
\label{section:threshold_models}
The strategy to prove Theorem~\ref{theorem:PB_constraints_hard_for_DNNF} is to find a PB-constraint $f$ over $n$ variables such that $C(f)$ is exponential in $\sqrt{n}$ and then use Theorem~\ref{theorem:size_rectangle_cover_size_dnnf}. We first show that we can restrict our attention to covering particular models of $f$ with rectangles rather than the whole function. In this section $X$ is a set of $n$ Boolean variables and $f$ is a PB-constraint over $X$. Recall that we only consider constraints of the form $\sum_{i = 1}^n w_i x_i \geq \theta$ where the $w_i$ and $\theta$ are positive integers.

\begin{definition}\label{definition:threshold_models}
The \emph{threshold models} of $f$ are the models $\mathbf{x}$ such that $w(\mathbf{x}) = \theta$. 
\end{definition}

\noindent Threshold models should not be confused with minimal models (or minimals).

\begin{definition}\label{definition:minimal_models}
A \emph{minimal} of $f$ is a model $\mathbf{x}$ such that no $\mathbf{y} < \mathbf{x}$ is a model of $f$.
\end{definition}  

\noindent For a monotone PB-constraint, a minimal model is such that its sum of weights drops below the threshold if we remove any element from it. Any threshold model is minimal, but not all minimals are threshold models. There even exist constraints with no threshold models (e.g. take even weights and an odd threshold) while there always are minimals for satisfiable constraints.

\begin{example}\label{example:third_example}
The minimal models of the PB-constraint from Example~\ref{example:PB_constraint_example} are $(0,0,0,1,1)$, $(0,1,1,1,0)$, $(1,0,1,0,1)$ and $(0,1,1,0,1)$. The first three are threshold models.
\end{example}

\noindent Let $f^*$ be the Boolean function whose models are exactly the threshold models of $f$. In the next lemma, we prove that the smallest rectangle cover of $f^*$ has size at most $C(f)$. Thus, lower bounds on $C(f^*)$ are also lower bounds on~$C(f)$.

\begin{lemma}\label{lemma:rectangle_cover_lower_bound_threshold_models}
Let $f^*$ be the Boolean function whose models are exactly the threshold models of $f$. Then $C(f) \geq C(f^*)$.
\end{lemma}
\begin{proof}
Let $r \coloneqq \rho_1 \wedge \rho_2$ be a balanced rectangle with $r \leq f$ and assume $r$ accepts some threshold models. Let $\Pi \coloneqq (X_1, X_2)$ be the partition of~$r$. We claim that there exist two integers $\theta_1$ and $\theta_2$ such that $\theta_1 + \theta_2 = \theta$ and, for any threshold model $\mathbf{x}$ accepted by~$r$, there is $w(\mathbf{x_1}) = \theta_1$ and  $w(\mathbf{x_2}) = \theta_2$. To see this, assume by contradiction that there exists another partition $\theta = \theta'_1 + \theta'_2$ of $\theta$ such that some other threshold model $\mathbf{y}$ with $w(\mathbf{y}_1) = \theta'_1$ and $w(\mathbf{y}_2) = \theta'_2$ is accepted by~$r$. Then either $w(\mathbf{x_1}) +  w(\mathbf{y_2}) < \theta$ or $w(\mathbf{y_1}) +  w(\mathbf{x_2}) < \theta$, but since $(\mathbf{x}_1,\mathbf{y}_2)$ and $(\mathbf{y}_1,\mathbf{x}_2)$ are also models of $r$, $r$ would accept a non-model of~$f$, which is forbidden. Now let $\rho^*_1$ (resp. $\rho^*_2$) be the function whose models are exactly the models of $\rho_1$ (resp. $\rho_2$) of weight $\theta_1$ (resp. $\theta_2$). Then $r^* \coloneqq \rho^*_1 \wedge \rho^*_2$ is a balanced rectangle whose models are exactly the threshold models accepted by $r$. 

Now consider a balanced rectangle cover of $f$ of size $C(f)$. For each rectangle~$r$ of the cover, if~$r$ accepts no threshold model then discard it, otherwise construct~$r^*$. The disjunction of these new rectangles is a balanced rectangle cover of $f^*$ of size at most $C(f)$. Therefore $C(f) \geq C(f^*)$.
\end{proof}

\section{Reduction to Covering Maximal Matchings of $K_{n,n}$}
\label{section:maximal_matchings}
We define the class of hard PB-constraints for Theorem~\ref{theorem:PB_constraints_hard_for_DNNF} in this section. Recall that for a hard constraint $f$, our aim is to find an exponential lower bound on $C(f)$. We will show, using Lemma~\ref{lemma:rectangle_cover_lower_bound_threshold_models}, that the problem can be reduced to that of covering  all maximal matchings of the complete $n \times n$ bipartite graph $K_{n,n}$ with rectangles. 
In this section, $X$ is a set of $n^2$ Boolean variables. For presentability reasons, assignments to $X$ are written as $n \times n$ matrices. Each variable $x_{i,j}$ has the weight $w_{i,j} \coloneqq (2^{i} + 2^{j+n})/2$. Define the matrix of weights $W \coloneqq \left(w_{i,j} : 1 \leq i,j \leq n \right)$ and the threshold $\theta \coloneqq 2^{2n}-1$. The PB-constraint $f$ for the pair $(W,\theta)$ is such that $f(\mathbf{x}) = 1$ if and only if $\mathbf{x}$ satisfies

\begin{equation}\label{equation:weird_function}
\sum_{1 \leq i,j \leq n}  \left(\frac{2^{i} + 2^{j+n}}{2}\right)x_{i,j} \geq  2^{2n}-1 \text{ .}
\end{equation}

Constraints of this form constitute the class of hard constraints of Theorem~\ref{theorem:PB_constraints_hard_for_DNNF}. One may find it easier to picture $f$ writing the weights and threshold as binary numbers of $2n$ bits. Bits of indices $1$ to $n$ form the \emph{lower part} of the number and those of indices $n+1$ to $2n$ form the \emph{upper part}. The weight $w_{i,j}$ is the binary number where the only bits set to $1$ are the $i$th bit of the lower part and the $j$th bit of the upper part. Thus when a variable $x_{i,j}$ is set to $1$, exactly one bit of value 1 is added to each part of the binary number of the sum.
\par Assignments to $X$ uniquely encode subgraphs of $K_{n,n}$. We denote $U = \{u_1, \dots, u_n\}$ the nodes of the left side and $V = \{v_1, \dots, v_n\}$ those of the right side of $K_{n,n}$. The bipartite graph encoded by $\mathbf{x}$ is such that there is an edge between the $u_i$ and $v_j$ if and only if $x_{i,j}$ is set to $1$ in $\mathbf{x}$. 

\begin{example}\label{example:bipartite_graphs_encoding}
Take $n = 4$. The assignment $
\mathbf{x} =
\footnotesize
\begin{pmatrix}
1 & 1 & 0 & 1\\
0 & 0 & 0 & 0\\
0 & 1 & 0 & 0\\
0 & 1 & 0 & 0\\
\end{pmatrix}
$ encodes 
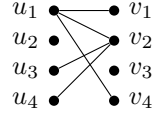
\raisebox{-0.45\height}{
\begin{tikzpicture}
\def\xgap{0.8};
\def\ygap{0.8};
\def\s{0.4};
\node[circle,fill=black,label=left:$u_1$,scale=\s,] (u1) at (0,+0.75*\ygap) {};
\node[circle,fill=black,label=left:$u_2$,scale=\s] (u2) at (0,+0.25*\ygap) {};
\node[circle,fill=black,label=left:$u_3$,scale=\s] (u3) at (0,-0.25*\ygap) {};
\node[circle,fill=black,label=left:$u_4$,scale=\s] (u4) at (0,-0.75*\ygap) {};
\node[circle,fill=black,label=right:$v_1$,scale=\s] (v1) at (\xgap,+0.75*\ygap) {};
\node[circle,fill=black,label=right:$v_2$,scale=\s] (v2) at (\xgap,+0.25*\ygap) {};
\node[circle,fill=black,label=right:$v_3$,scale=\s] (v3) at (\xgap,-0.25*\ygap) {};
\node[circle,fill=black,label=right:$v_4$,scale=\s] (v4) at (\xgap,-0.75*\ygap) {};
\draw (u1) -- (v1);
\draw (u1) -- (v2);
\draw (u1) -- (v4);
\draw (u3) -- (v2);
\draw (u4) -- (v2);
\end{tikzpicture}
}
\end{example}

\begin{definition}\label{definition:maximal_matching_assignment}
A \emph{maximal matching assignment} (or maximal matching model) is an assignment $\mathbf{x}$ to $X$ such that 
\begin{enumerate}[topsep=0pt]
\item[$\bullet$] for any $i \in [n]$, there is exactly one $k$ such that $x_{i,k}$ is set to $1$ in $\mathbf{x}$,
\item[$\bullet$] for any $j \in [n]$, there is exactly one $k$ such that $x_{k,j}$ is set to $1$ in $\mathbf{x}$. 
\end{enumerate}
\end{definition}
As the name suggests, the maximal matching assignments are those encoding graphs whose edges form a maximal matching of $K_{n,n}$ (i.e., a maximum cardinality matching). One can also see them as encodings for permutations of $[n]$.\begin{example}\label{example:maximal_matching_assignment}
The maximal matching model $
\mathbf{x} =
\footnotesize
\begin{pmatrix}
0 & 0 & 1 & 0\\
1 & 0 & 0 & 0\\
0 & 0 & 0 & 1\\
0 & 1 & 0 & 0\\
\end{pmatrix}
$ encodes
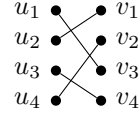
\raisebox{-0.45\height}{
\begin{tikzpicture}
\def\xgap{0.6};
\def\ygap{0.8};
\def\s{0.4};
\node[circle,fill=black,label=left:$u_1$,scale=\s] (u1) at (0,+0.75*\ygap) {};
\node[circle,fill=black,label=left:$u_2$,scale=\s] (u2) at (0,+0.25*\ygap) {};
\node[circle,fill=black,label=left:$u_3$,scale=\s] (u3) at (0,-0.25*\ygap) {};
\node[circle,fill=black,label=left:$u_4$,scale=\s] (u4) at (0,-0.75*\ygap) {};
\node[circle,fill=black,label=right:$v_1$,scale=\s] (v1) at (\xgap,+0.75*\ygap) {};
\node[circle,fill=black,label=right:$v_2$,scale=\s] (v2) at (\xgap,+0.25*\ygap) {};
\node[circle,fill=black,label=right:$v_3$,scale=\s] (v3) at (\xgap,-0.25*\ygap) {};
\node[circle,fill=black,label=right:$v_4$,scale=\s] (v4) at (\xgap,-0.75*\ygap) {};
\draw (u1) -- (v3);
\draw (u2) -- (v1);
\draw (u3) -- (v4);
\draw (u4) -- (v2);
\end{tikzpicture}
}
\end{example}

For a given $\mathbf{x}$, define $\varl{k}{\mathbf{x}}$ by $\varl{k}{\mathbf{x}} \coloneqq \{ j \mid x_{k,j} \text{ is set to } 1 \text{ in } \mathbf{x}\}$ when $1 \leq k \leq n$ and by $\varl{k}{\mathbf{x}} \coloneqq \{ i \mid x_{i,k-n} \text{ is set to } 1 \text{ in } \mathbf{x}\}$ when $n+1 \leq k \leq 2n$. $\varl{k}{\mathbf{x}}$ stores the index of variables in $\mathbf{x}$ that directly add $1$ to the $k$th bit of $w(\mathbf{x})$. Note that a maximal matching model is an assignment $\mathbf{x}$ such that $\card{\varl{k}{\mathbf{x}}} = 1$ for all $k$. It is easy to see that maximal matching models are threshold models of $f$: seeing weights as binary numbers of $2n$ bits, for every bit of the sum the value 1 is added exactly once, so exactly the first $2n$ bits of the sum are set to 1, which gives us $\theta$. Note that not all threshold models of $f$ are maximal matching models, for instance the assignment from Example~\ref{example:bipartite_graphs_encoding} does not encode a maximal matching but one can verify that it is a threshold model. Recall that $f^*$ is the function whose models are the threshold models of $f$. In the next lemmas, we prove that lower bounds on the size of rectangle covers of the maximal matching models are lower bounds on $C(f^*)$, and a fortiori on $C(f)$.

\begin{lemma}\label{lemma:rectangle_property_of_maximal_matching}
Let $\Pi \coloneqq (X_1, X_2)$ be a partition of $X$. Let $\mathbf{x} \coloneqq (\mathbf{x}_1, \mathbf{x}_2)$ and $\mathbf{y} \coloneqq (\mathbf{y}_1, \mathbf{y}_2)$ be maximal matching assignments. If $(\mathbf{x}_1, \mathbf{y}_2)$ and $(\mathbf{y}_1, \mathbf{x}_2)$ both have weight $\theta \coloneqq 2^{2n}-1$ then both are maximal matching assignments.
\end{lemma}
\begin{proof}
It is sufficient to show that $\card{\varl{k}{\mathbf{x}_1,\mathbf{y}_2}} = 1$ and $\card{\varl{k}{\mathbf{y}_1,\mathbf{x}_2}} = 1$ for all $1 \leq k \leq 2n$. We prove it for $(\mathbf{x}_1,\mathbf{y}_2)$ by induction on $k$. First observe that since $\card{\varl{k}{\mathbf{x}}} = 1$ and $\card{\varl{k}{\mathbf{y}}} = 1$ for all $1 \leq k \leq 2n$, the only possibilities for $\card{\varl{k}{\mathbf{x}_1,\mathbf{y}_2}}$ are $0$, $1$ or $2$.
\begin{enumerate}[topsep=0pt]
\item[$\bullet$] For the base case $k = 1$, if $\card{\varl{1}{\mathbf{x}_1,\mathbf{y}_2}}$ is even then the first bit of $w(\mathbf{x}_1) + w(\mathbf{y}_2)$ is $0$ and the weight of $(\mathbf{x}_1,\mathbf{y}_2)$ is not $\theta$. So $\card{\varl{1}{\mathbf{x}_1,\mathbf{y}_2}} = 1$.
\item[$\bullet$]  For the general case $1 < k \leq 2n$, assume it holds that $\card{\varl{1}{\mathbf{x}_1,\mathbf{y}_2}} = \dots = \card{\varl{k-1}{\mathbf{x}_1,\mathbf{y}_2}} = 1$. So the $k$th bit of $w(\mathbf{x}_1) + w(\mathbf{y}_2)$ depends only on the parity of $\card{\varl{k}{\mathbf{x}_1,\mathbf{y}_2}}$: the $k$th bit is $0$ if $\card{\varl{k}{\mathbf{x}_1,\mathbf{y}_2}}$ is even and $1$ otherwise. $(\mathbf{x}_1,\mathbf{y}_2)$ has weight $\theta$ so $\card{\varl{k}{\mathbf{x}_1,\mathbf{y}_2}} = 1$.
\end{enumerate}
The argument applies to $(\mathbf{y}_1, \mathbf{x}_2)$ analogously. 
\end{proof}

\begin{lemma}\label{lemma:rectangle_cover_lower_bound_maximal_matching_assignments}
Let $f$ be the PB-constraint (\ref{equation:weird_function}) and let $\hat{f}$ be the function whose models are exactly the maximal matching assignments. Then $C(f) \geq C(\hat{f})$.
\end{lemma}
\begin{proof}
By Lemma~\ref{lemma:rectangle_cover_lower_bound_threshold_models}, it is sufficient to prove that $C(f^*) \geq C(\hat{f})$. We already know that $\hat{f} \leq f^*$. Let $r \coloneqq \rho_1 \wedge \rho_2$ be a balanced rectangle of partition $\Pi \coloneqq (X_1, X_2)$ with $r \leq f^*$, and assume $r$ accepts some maximal matching assignment. Let $\hat{\rho}_1$ (resp. $\hat{\rho}_2$) be the Boolean function over $X_1$ (resp. $X_2$) whose models are the $\mathbf{x}_1$ (resp. $\mathbf{x}_2$) such that there is a maximal matching assignment $(\mathbf{x}_1, \mathbf{x}_2)$ accepted by $r$. We claim that the balanced rectangle $\hat{r} \coloneqq \hat{\rho}_1 \wedge \hat{\rho}_2$ accepts exactly the maximal matching models of $r$. On the one hand, it is clear that all maximal matching models of $r$ are models of $\hat{r}$. On the other hand, all models of $\hat{r}$ are threshold models of the form $(\mathbf{x}_1, \mathbf{y}_2)$, where $(\mathbf{x}_1, \mathbf{x}_2)$ and $(\mathbf{y}_1, \mathbf{y}_2)$  encode maximal matchings, so by Lemma~\ref{lemma:rectangle_property_of_maximal_matching}, $\hat{r}$ accepts only maximal matching models~of~$r$.

Now consider a balanced rectangle cover of $f^*$ of size $C(f^*)$. For each rectangle~$r$ of the cover, if $r$ accepts no maximal matching assignment then discard it, otherwise construct $\hat{r}$. The disjunction of these new rectangles is a balanced rectangle cover of $\hat{f}$ of size at most $C(f^*)$. Therefore $C(f^*) \geq C(\hat{f})$.
\end{proof}

\section{Proof of Theorem~\ref{theorem:PB_constraints_hard_for_DNNF}}
\label{section:proof_main_result}
\setcounter{theorem}{0}
\begin{theorem}
There is a class of PB-constraints $\mathcal{F}$ such that for any constraint $f \in \mathcal{F}$ on $n^2$ variables, any {\DNNF} encoding $f$ has size $2^{\Omega(n)}$.
\end{theorem}
\noindent $\mathcal{F}$ is the class of constraints defined in (\ref{equation:weird_function}). Thanks to Theorem~\ref{theorem:size_rectangle_cover_size_dnnf} and Lemma~\ref{lemma:rectangle_cover_lower_bound_maximal_matching_assignments}, the proof boils down to finding exponential lower bounds on $C(\hat{f})$, where $\hat{f}$ is the Boolean function on $n^2$ variables whose models encode exactly the maximal matchings of $K_{n,n}$ (or equivalently, the permutations of $[n]$). $\hat{f}$ has $n!$ models. The idea is now to prove that rectangles covering $\hat{f}$ must be relatively small, so that covering the whole function requires many of them.

\begin{lemma}\label{lemma:upper_bound_size_rectangle_for_maximal_matching_assignments}
Let $\Pi = (X_1, X_2)$ be a balanced partition of $X$. Let $r$ be a rectangle with respect to $\Pi$ with $r \leq \hat{f}$. Then $\card{\inv{r}(1)} \leq  n!/\binom{n}{n\sqrt{2/3}}$.
\end{lemma}
\noindent The function $\hat{f}$ has already been studied extensively in the literature, often under the name $\text{PERM}_n$ (for \emph{permutations on $[n]$}), see for instance Chapter 4 of \cite{Wegener00} or section 6.2 of \cite{MengelW19} where a statement similar to Lemma~\ref{lemma:upper_bound_size_rectangle_for_maximal_matching_assignments} is established. With Lemma~\ref{lemma:upper_bound_size_rectangle_for_maximal_matching_assignments} we can give the proof of Theorem~\ref{theorem:PB_constraints_hard_for_DNNF}.
\begin{proof}[Theorem~\ref{theorem:PB_constraints_hard_for_DNNF}]
Let $\bigvee_{k = 1}^{C(\hat{f})} r_k$ be a balanced rectangle cover of $\hat{f}$. We have $\sum_{k = 1}^{C(\hat{f})} \card{\inv{r_k}(1)} \geq \card{\inv{\hat{f}}(1)} = n!$. Lemma~\ref{lemma:upper_bound_size_rectangle_for_maximal_matching_assignments} gives us $(C(\hat{f})n!)/\binom{n}{n\sqrt{2/3}} \geq n!$, thus
\begin{equation*}
\begin{aligned}
C(\hat{f}) \geq \binom{n}{n\sqrt{2/3}} \geq \left(\frac{n}{n\sqrt{2/3}}\right)^{n\sqrt{2/3}} 
= \left(\frac{3}{2}\right)^{n\frac{\sqrt{2/3}}{2}} \geq 2^{n\frac{\sqrt{2/3}}{4}} = 2^{\Omega(n)}
\end{aligned}
\end{equation*}
where we have used $\binom{a}{b} \geq (a/b)^b$ and $3/2 \geq \sqrt{2}$. Using Lemma~\ref{lemma:rectangle_cover_lower_bound_maximal_matching_assignments} we get that $C(f) \geq C(\hat{f}) \geq 2^{\Omega(n)}$. Theorem~\ref{theorem:size_rectangle_cover_size_dnnf} allows us to conclude.
\end{proof}

\begin{figure}[t]
\hfill
\begin{subfigure}[t]{0.4\textwidth}
\centering
\begin{tikzpicture}
\def\xgap{0.8};
\def\ygap{-0.55};
\def\offsetx{0}
\def\offsety{0}
\def\s{0.4};
\node[circle,fill=black,label=left:$u_1$,scale=\s] (u1) at (\offsetx,0*\ygap+\offsety) {};
\node[circle,fill=black,label=left:$u_2$,scale=\s] (u2) at (\offsetx,1*\ygap+\offsety) {};
\node[circle,fill=black,label=left:$u_3$,scale=\s] (u3) at (\offsetx,2*\ygap+\offsety) {};
\node[circle,fill=black,label=left:$u_4$,scale=\s] (u4) at (\offsetx,3*\ygap+\offsety) {};
\node[circle,fill=black,label=right:$v_1$,scale=\s] (v1) at (\xgap+\offsetx,0*\ygap+\offsety) {};
\node[circle,fill=black,label=right:$v_2$,scale=\s] (v2) at (\xgap+\offsetx,1*\ygap+\offsety) {};
\node[circle,fill=black,label=right:$v_3$,scale=\s] (v3) at (\xgap+\offsetx,2*\ygap+\offsety) {};
\node[circle,fill=black,label=right:$v_4$,scale=\s] (v4) at (\xgap+\offsetx,3*\ygap+\offsety) {};
\draw[densely dotted] (u1) -- (v1);
\draw (u1) -- (v2);
\draw (u1) -- (v3);
\draw[densely dotted] (u1) -- (v4);

\draw[densely dotted] (u2) -- (v1);
\draw[densely dotted] (u2) -- (v2);
\draw (u2) -- (v3);
\draw (u2) -- (v4);

\draw (u3) -- (v1);
\draw (u3) -- (v2);
\draw[densely dotted] (u3) -- (v3);
\draw (u3) -- (v4);

\draw (u4) -- (v1);
\draw[densely dotted] (u4) -- (v2);
\draw[densely dotted] (u4) -- (v3);
\draw[densely dotted] (u4) -- (v4);

\node (white) at (\xgap+\offsetx,3*\ygap-0.775) {};
\end{tikzpicture}
\caption{Balanced partition $\Pi$ of $K_{4,4}$} \label{figure:balanced_partition_Knn}
\end{subfigure}\hfill
\begin{subfigure}[t]{0.6\textwidth}
\centering
\begin{tikzpicture}
\def\xgap{0.8};
\def\ygap{-0.55};
\def\s{0.4};
\node[circle,fill=black,label=left:$u_1$,scale=\s] (u1) at (0,0*\ygap) {};
\node[circle,fill=black,label=left:$u_2$,scale=\s] (u2) at (0,1*\ygap) {};
\node[circle,fill=black,label=left:$u_3$,scale=\s] (u3) at (0,2*\ygap) {};
\node[circle,fill=black,label=left:$u_4$,scale=\s] (u4) at (0,3*\ygap) {};
\node[circle,fill=black,label=right:$v_1$,scale=\s] (v1) at (\xgap,0*\ygap) {};
\node[circle,fill=black,label=right:$v_2$,scale=\s] (v2) at (\xgap,1*\ygap) {};
\node[circle,fill=black,label=right:$v_3$,scale=\s] (v3) at (\xgap,2*\ygap) {};
\node[circle,fill=black,label=right:$v_4$,scale=\s] (v4) at (\xgap,3*\ygap) {};
\draw (u1) -- (v3);
\draw (u2) -- (v1);
\draw (u3) -- (v4);
\draw (u4) -- (v2);

\draw[-stealth] (\xgap+0.75,1*\ygap-0.1) -- (\xgap+2.5,-0.15) node[midway,above] {$\mathbf{x}_1$} ;
\draw[-stealth] (\xgap+0.75,2*\ygap+0.1) -- (\xgap+2.5,-0.15-1.5)  node[midway,below] {$\mathbf{x}_2$} ;

\def\xgap{0.8};
\def\ygap{-0.3};
\def\offsetx{4}
\def\offsety{0.5}
\def\s{0.4};
\node[circle,fill=black,label=left:$ $,scale=\s] (u1) at (\offsetx,0*\ygap+\offsety) {};
\node[circle,fill=black,label=left:$ $,scale=\s] (u2) at (\offsetx,1*\ygap+\offsety) {};
\node[circle,fill=black,label=left:$ $,scale=\s] (u3) at (\offsetx,2*\ygap+\offsety) {};
\node[circle,fill=black,label=left:$ $,scale=\s] (u4) at (\offsetx,3*\ygap+\offsety) {};
\node[circle,fill=black,label=right:$ $,scale=\s] (v1) at (\xgap+\offsetx,0*\ygap+\offsety) {};
\node[circle,fill=black,label=right:$ $,scale=\s] (v2) at (\xgap+\offsetx,1*\ygap+\offsety) {};
\node[circle,fill=black,label=right:$ $,scale=\s] (v3) at (\xgap+\offsetx,2*\ygap+\offsety) {};
\node[circle,fill=black,label=right:$ $,scale=\s] (v4) at (\xgap+\offsetx,3*\ygap+\offsety) {};
\draw (u1) -- (v3);
\draw (u3) -- (v4);

\draw let \p{A}=(u1), \p{B}=(u3) in [densely dotted]  plot [smooth cycle] coordinates {(\x{A} + 3,\y{A}) (\x{A}-1, \y{A} + 3) (\x{A}-8,\y{A}*0.5+\y{B}*0.5) (\x{B}-1, \y{B}-3) (\x{B} + 3,\y{B}) (\x{A}-3,\y{A}*0.5+\y{B}*0.5) };
\node[scale=0.8] (n) at (\xgap+\offsetx-1.3,1*\ygap+\offsety) {$U_1$};

\draw[densely dotted] (\xgap+\offsetx,2.5*\ygap+\offsety) ellipse (0.15 and 0.32);
\node[scale=0.8] (n) at (\xgap+\offsetx+0.4,2.5*\ygap+\offsety) {$V_1$};

\def\xgap{0.8};
\def\ygap{-0.3};
\def\offsetx{4}
\def\offsety{-1.49}
\def\s{0.4};
\node[circle,fill=black,label=left:$ $,scale=\s] (u1) at (\offsetx,0*\ygap+\offsety) {};
\node[circle,fill=black,label=left:$ $,scale=\s] (u2) at (\offsetx,1*\ygap+\offsety) {};
\node[circle,fill=black,label=left:$ $,scale=\s] (u3) at (\offsetx,2*\ygap+\offsety) {};
\node[circle,fill=black,label=left:$ $,scale=\s] (u4) at (\offsetx,3*\ygap+\offsety) {};
\node[circle,fill=black,label=right:$ $,scale=\s] (v1) at (\xgap+\offsetx,0*\ygap+\offsety) {};
\node[circle,fill=black,label=right:$ $,scale=\s] (v2) at (\xgap+\offsetx,1*\ygap+\offsety) {};
\node[circle,fill=black,label=right:$ $,scale=\s] (v3) at (\xgap+\offsetx,2*\ygap+\offsety) {};
\node[circle,fill=black,label=right:$ $,scale=\s] (v4) at (\xgap+\offsetx,3*\ygap+\offsety) {};
\draw (u2) -- (v1);
\draw (u4) -- (v2);

\draw let \p{A}=(u2), \p{B}=(u4) in [densely dotted]  plot [smooth cycle] coordinates {(\x{A} + 3,\y{A}) (\x{A}-1, \y{A} + 3) (\x{A}-8,\y{A}*0.5+\y{B}*0.5) (\x{B}-1, \y{B}-3) (\x{B} + 3,\y{B}) (\x{A}-3,\y{A}*0.5+\y{B}*0.5) };
\node[scale=0.8] (n) at (\xgap+\offsetx-1.3,2*\ygap+\offsety) {$U_2$};

\draw[densely dotted] (\xgap+\offsetx,0.5*\ygap+\offsety) ellipse (0.15 and 0.32);
\node[scale=0.8] (n) at (\xgap+\offsetx+0.4,0.5*\ygap+\offsety) {$V_2$};
\end{tikzpicture}
\caption{Partition of a maximal matching w.r.t. $\Pi$} \label{figure:balanced_partition_matching}
\end{subfigure}
\caption{Partition of maximal matching}\label{figure:fig}
\end{figure}
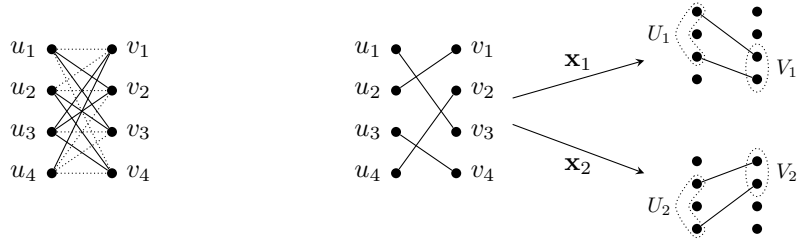

\noindent All that is left is to prove Lemma~\ref{lemma:upper_bound_size_rectangle_for_maximal_matching_assignments}.
\begin{proof}[Lemma~\ref{lemma:upper_bound_size_rectangle_for_maximal_matching_assignments}]

Let $r \coloneqq \rho_1 \wedge \rho_2$ and $\Pi \coloneqq (X_1,X_2)$. Recall that $U \coloneqq \{u_1, \dots, u_n\}$ and $V \coloneqq \{v_1, \dots, v_n\}$ are the nodes from the left and right part of $K_{n,n}$ respectively. Define $U_1 \coloneqq \{u_i \mid \text{there exists } x_{i,l} \in X_1 \text{ such that  a model of } \rho_1 \text{ has } x_{i,l} \text{ set to } 1\}$ and $V_1 \coloneqq \{v_j \mid \text{there exists } x_{l,j} \in X_1 \text{ such that  a model of } \rho_1 \text{ has } x_{l,j} \text{ set to } 1\}$. Define $U_2$ and $V_2$ analogously (this time using $X_2$ and $\rho_2$). Figure~\ref{figure:fig} illustrates the construction of these sets: Figure~\ref{figure:balanced_partition_Knn} shows a partition $\Pi$ of the edges of $K_{4,4}$ (full edges in $X_1$, dotted edges in $X_2$) and Figure~\ref{figure:balanced_partition_matching} shows the contribution of a model of $r$ to $U_1$, $V_1$, $U_2$, and $V_2$ after partition according to $\Pi$.

Models of $\rho_1$ are clearly matchings of $K_{n,n}$. Actually they are matchings between $U_1$ and $V_1$ by construction of these sets. We claim that they are maximal. To verify this, observe that $U_1 \cap U_2 = \emptyset$ and $V_1 \cap V_2 = \emptyset$ since otherwise $r$ has a model that is not a matching. Thus if $\rho_1$ were to accept a non-maximal matching between $U_1$ and $V_1$ then $r$ would accept a non-maximal matching between $U$ and $V$. So $\rho_1$ accepts only maximal matchings between $U_1$ and $V_1$, consequently $\card{U_1} = \card{V_1}$. The argument applies symmetrically for $V_2$ and $U_2$. We note $k \coloneqq \card{U_1}$.
It stands that $U_1 \cup U_2 = U$ and $V_1 \cup V_2 = V$ as otherwise $r$ accepts matchings that are not maximal. So $\card{U_2} = \card{V_2} = n-k$. We now have $\card{\inv{\rho_1}(1)} \leq k!$ and $\card{\inv{\rho_2}(1)} \leq (n-k)!$, leading to $\card{\inv{r}(1)} \leq k!(n-k)! = n!/\binom{n}{k}$. 

Up to $k^2$ edges may be used to build matchings between $U_1$ and $V_1$. Since~$r$ is balanced we obtain $k^2 \leq 2n^2/3$. Applying the same argument to $U_2$ and $V_2$ gives us $(n-k)^2 \leq 2n^2/3$, so $n(1-\sqrt{2/3}) \leq k \leq n\sqrt{2/3}$. Finally, the function $k \mapsto n!/\binom{n}{k}$, when restricted to some interval $[\![n(1-\alpha), \alpha n]\!]$, reaches its maximum at $k = \alpha n$, hence the upper bound $\card{\inv{r}(1)} \leq  n!/\binom{n}{n\sqrt{2/3}}$. 
\end{proof}

\section*{Acknowledgments}
This work has been partly supported by the PING/ACK project of the French National Agency for Research (ANR-18-CE40-0011).

\newpage
\bibliography{biblio}
\bibliographystyle{abbrv}

\end{document}